\documentclass[runningheads]{llncs}

\usepackage[T1]{fontenc}
\usepackage{amsmath,amssymb,amsfonts}
\usepackage{mathtools}
\usepackage{graphicx}
\usepackage{booktabs}
\usepackage{array}
\usepackage{hyperref}
\usepackage{cite}
\usepackage[protrusion=true,expansion=false]{microtype}
\usepackage{xcolor}
\usepackage{comment}

\newcommand{\dFI}{d_{\mathrm{FI}}}
\newcommand{\dPMS}{d_{\mathrm{PMS}}}
\newcommand{\dVC}{d_{\mathrm{VC}}}
\newcommand{\dNat}{d_{\mathrm{Nat}}}

\newcommand{\Fstar}{F^{*}}
\newcommand{\Coh}{\mathcal{C}}

\newcommand{\VM}{\text{VM}}
\newcommand{\PMS}{\text{PMS}}
\newcommand{\RL}{\mathrm{RL}}
\newcommand{\Ldim}{\mathrm{Ldim}}

\newcommand{\eps}{\varepsilon}
\DeclareMathOperator*{\argmax}{arg\,max}

\begin{document}

\title{The Free Inference Dimension:\\
A Complexity Measure for Zero-Collision\\
Navigation under Hypothesis Mixtures}

% Double-anonymous submission: author information removed for review.
\titlerunning{The Free Inference Dimension}
%
%\author{Anominous}
%\begin{comment}
\author{%
Luiz Carlos Castro Guedes\inst{1,2}\orcidID{0009-0005-2405-2379} \and \\
Edward Hermann Haeusler\inst{1}\orcidID{0000-0002-4999-7476}\thanks{Edward H. Haeusler is p.f by FAPERJ grant APQ1 E-26/210.258/2019.249292), CNPq grant 309287/2023-5 and CAPES/COFECUB 88881.878969/2023-01}%
}
\authorrunning{L.C.C. Guedes and E.H. Haeusler}
\institute{Departamento de Inform\'atica -- PUC-Rio, Brazil \and
Se\c{c}\~ao de Engenharia de Computa\c{c}\~ao -- Instituto Militar de Engenharia, Brazil\\
\email{lcguedes@ime.eb.br, hermann@inf.puc-rio.br}}
%\end{comment}
%

\maketitle

\begin{abstract}

Solomonoff induction frames prediction as a mixture over computable hypotheses, typically leading to identification of the true environment. In a finite meta-reinforcement learning setting with nested constraint families, in our previous work, we observe a different regime: a value-mixture (VM) agent achieves near-optimal, zero-collision navigation without identifying the true environment, a phenomenon we call Free Inference. This regime persists up to a sharp density threshold, beyond which performance degrades and posterior-mode selection (PMS) becomes preferable.

We formalize this behavior via the Free Inference dimension \(d_{FI}(S,N)\), a combinatorial measure of the environmental complexity a VM agent can handle while preserving trajectory coherence. We prove \(d_{FI}\) is strictly smaller than the VC-dimension and relates to the Natarajan dimension up to a path-length factor, capturing the cost of non-decomposable loss. A PAC-style relaxation yields generalization bounds driven by \(d^{\varepsilon,\delta}_{FI}\). We also define a complementary PMS identification dimension and show that a hybrid strategy---averaging until the first collision, then switching to selection---is optimal, with links to Littlestone-type dimensions supported by grid-world experiments.

\end{abstract}

%=========================================================================
\section{Introduction}
%=========================================================================

A central design question in meta-reinforcement learning~\cite{humplik2019,duan2016,rakelly2019,zintgraf2021} is how an agent should exploit
a finite Bayesian belief over candidate environments. Two canonical answers
stand out. The \emph{value-mixture (VM) agent} averages the optimal
$Q$-functions $Q_1, \ldots, Q_N$ of the $N$ hypothesized environments under
a belief $\xi$ and acts greedily with respect to
$\bar{Q} = \sum_i \xi(\nu_i) Q_i$; this is a finite, constructive
approximation of a Solomonoff-style universal mixture~\cite{solomonoff64,hutter2005}
restricted to a hypothesis class of environment models, and it corresponds
to the risk-neutral Bayes-adaptive control rule~\cite{ghavamzadeh15}.
The \emph{posterior mode sampling (PMS) agent}, by contrast, selects the
maximum-a-posteriori (MAP) environment at each step and follows its policy,
updating the belief only upon collision. Averaging versus selection is the
defining tension in action under posterior uncertainty, and the two
strategies have shown on our experiments to behave very differently in the action-selection
setting even though they use the same posterior.

On a recent work on this setting we have identified a
striking phase transition in the relative performance of the two modes.
When the candidate environments form a \emph{nested} family, where environments differ in
how many cells are classified as obstacles and more restrictive
environments contain all the obstacles of less restrictive ones, the VM
agent was shown to navigate from a start state to a goal state
\emph{without any collisions and without identifying the true environment},
up to a sharp density threshold. We have termed that regime as \emph{Free
Inference}, the regime where the agent collects value without paying the usual
identification cost, while the belief remains approximately uniform
throughout the episode. Beyond a certain density threshold, the mixture becomes
navigationally incoherent, the agent oscillates or collides, and PMS
begins to outperform VM. Empirically the boundary is remarkably stable:
$\rho_c \approx 0.7$ across grid sizes ranging from
$10 \times 5$ to $100 \times 50$.

The empirical behaviour of this threshold has the hallmarks of a
combinatorial capacity measure. It is a sharp threshold transition,
not a gradual one; it scales linearly with the state space; and it
separates a regime in which the learner is reliably correct from one in
which it is reliably wrong. Yet it does not match any existing dimension
in the learning-theoretic landscape, for two structural reasons. First,
the loss function (trajectory success) is \emph{non-decomposable}: local
correctness at every state is necessary but not sufficient for global
success, so Vapnik-Chervonenkis- and Natarajan-style pointwise arguments do not apply~\cite{vapnik1974pattern,natarajan89}.
Second, the adversary is \emph{non-adaptive}: the environment is committed
before the episode begins, in contrast to the adaptive adversary assumed
by Littlestone-style mistake-bound dimensions. This paper develops the
complexity-theoretic foundations of Free Inference, taking the empirical
phase transition as a starting point. Our central question is:
\emph{what is the combinatorial dimension underlying Free Inference, and
how does it sit within the existing hierarchy?}

\medskip\noindent\textbf{Contributions.} We make the following contributions.
\begin{enumerate}
\item We introduce the \emph{Free Inference dimension} $\dFI(\mathcal{S}, N)$,
  defined as the maximum size of a set of constraint cells FI-shatterable
  under all $N$-level nesting assignments (Definition~\ref{def:dfi}).
\item We show $\dFI < \dVC$ for $N \geq 2$, and quantify the ``coherence gap''
  as $\Delta_{\text{coh}} \approx 0.3 \, |\mathcal{S}|$ empirically
  (Section~\ref{sec:classical}).
\item We formulate a PAC-style relaxation ($\eps$-Free Inference) and prove a
  generalization bound in which $\dFI^{\eps,\delta}$ controls the sample
  complexity of certifying coherent navigation on unseen nesting schemes
  (Theorem~\ref{thm:generalization}).
\item We introduce the complementary \emph{PMS identification dimension}
  $\dPMS \leq N-1$ and prove that the hybrid strategy
  ``VM until first collision, then PMS'' is optimal in a precise sense
  (Theorem~\ref{thm:hybrid}).
\item We position $\dFI$ with respect to the Randomized Littlestone dimension
  of Filmus \emph{et al.}~\cite{filmus2025} and the Stackelberg--Littlestone
  dimension of Balcan \emph{et al.}~\cite{balcan2026}, identifying
  \emph{non-adaptivity} as the structural feature that enables the
  exponential gap $\dFI = \Theta(|\mathcal{S}|)$ vs.\ $\Ldim = O(\log N)$.
\item We describe a concrete experimental protocol on grid-world benchmarks
  (Section~\ref{sec:experiments}) designed to characterize the scaling of
  $\dFI(N)$ with the family size $N$ and validate the generalization bound.
\end{enumerate}

\medskip\noindent\textbf{Paper organization.} Section~\ref{sec:setting}
fixes the abstract setting. Section~\ref{sec:dimension} defines $\dFI$ via
an adversarial shattering game and establishes elementary properties.
Section~\ref{sec:classical} relates $\dFI$ to VC, Littlestone and Natarajan
dimensions. Section~\ref{sec:pac} develops the $\eps$-relaxation and proves
the generalization bound. Section~\ref{sec:pms} introduces $\dPMS$ and the
VM--PMS duality. Section~\ref{sec:learning-theory} summarizes the connection
to recent online-learning results. Section~\ref{sec:experiments} presents
the experimental protocol, and Section~\ref{sec:conclusion} concludes with
open problems.

%=========================================================================
\section{Abstract Setting}\label{sec:setting}
%=========================================================================

\subsection{Environments and Constraints}

\begin{definition}[Structured decision environment]
A \emph{structured decision environment} is a tuple
$\nu = (\mathcal{S}, \mathcal{A}, T_\nu, R_\nu, s_0, g)$ where
$\mathcal{S}$ and $\mathcal{A}$ are finite state and action spaces,
$T_\nu : \mathcal{S} \times \mathcal{A} \to \mathcal{S}$ is deterministic,
$R_\nu$ has asymmetric costs ($|r_{\text{penalty}}| \gg |r_{\text{step}}|$),
and $s_0, g \in \mathcal{S}$ are shared start and goal states. Each $\nu$
defines a \emph{constraint set} $O_\nu \subset \mathcal{S}$ (obstacles)
and a \emph{feasible set} $F_\nu = \mathcal{S} \setminus O_\nu$.
\end{definition}

A family $\mathcal{V} = \{\nu_1,\ldots,\nu_N\}$ is \emph{nested} if
$O_{\nu_1} \subseteq \cdots \subseteq O_{\nu_N}$. The
\emph{consensus feasible set} is
$\Fstar = \bigcap_{\nu \in \mathcal{V}} F_\nu$.

\begin{definition}[Entry-level assignment]\label{def:entry}
Given a set $C \subseteq \mathcal{S} \setminus \{s_0,g\}$ of \emph{constraint
cells} and $N$ levels, an \emph{entry-level assignment} is a map
$\ell: C \to \{1,\ldots,N\}$. The induced nested family
$\mathcal{V}_{C,\ell}$ has obstacle sets
$O_k = \{c \in C : \ell(c) \leq k\}$. The number of such assignments is
$N^{|C|}$.
\end{definition}

\subsection{Value-Mixture Agent and Free Inference}

Given expert action-value functions $\{Q_i\}_{i=1}^N$, one per environment,
and belief $\xi : \mathcal{V} \to [0,1]$, the \emph{value-mixture policy} is
\[
\pi_{\VM}(s) = \argmax_{a \in \mathcal{A}}\; \sum_{i=1}^{N} \xi(\nu_i)\, Q_i(s, a).
\]
Under uniform belief $\xi(\nu_i) = 1/N$ this reduces to
$\pi_{\VM}(s) = \argmax_a \bar{Q}(s,a)$.

\begin{definition}[Free Inference]
The VM agent achieves \emph{Free Inference} on $\mathcal{V}$ if, under
uniform belief, the trajectory $(s_0, s_1, \ldots, s_T)$ generated by
$\pi_{\VM}$ satisfies (i) $s_T = g$ and
(ii) $s_t \in \Fstar$ for all $t$.
\end{definition}

Free Inference is a \emph{zero-regret, zero-collision} property: the agent
reaches the goal while remaining in cells that are free in \emph{every}
candidate environment.

\begin{remark}[Averaging as a worst-case guarantee]\label{rem:worstcase}
Under uniform belief $\bar Q(s,a)=\tfrac1N\sum_i Q_i(s,a)$, and since penalties
dominate step costs ($|r_{\text{penalty}}|\gg|r_{\text{step}}|$), any action
colliding in even one hypothesis incurs a term $\approx r_{\text{penalty}}/N$
that dominates the step-cost spread from the others. Thus $\argmax_a\bar Q$
vetoes any action unsafe in \emph{any} $\nu_i$: greedy-on-average acts as a soft
max--min over the family, which is why averaging, not identification,
yields a guarantee holding across all hypotheses at once. It also shows why the
loss is \emph{non-decomposable}: at a fork where every action is locally
feasible, $\bar Q$ may steer into a consensus-free pocket whose only exit is a
contested cell; the veto then blocks that exit, so the agent never collides yet
never reaches $g$. Local correctness held everywhere; global coherence (C2) did not.
\end{remark}

%=========================================================================
\section{The Free Inference Dimension}\label{sec:dimension}
%=========================================================================

\subsection{The Shattering Game}

We define $\dFI$ through an adversarial game between a \emph{Constructor},
who designs the constraint structure, and the VM agent, which must navigate
coherently.

\begin{definition}[FI-shattering game]
The game $\mathcal{G}_{\text{FI}}(\mathcal{S}, N)$ proceeds as follows:
(1) the Constructor selects $C \subseteq \mathcal{S} \setminus \{s_0,g\}$ of
size $d$; (2) the Constructor selects an entry-level assignment
$\ell: C \to \{1,\ldots,N\}$, inducing the family $\mathcal{V}_{C,\ell}$
(Definition~\ref{def:entry}); (3) the VM
computes the mixture policy $\pi_{\VM}$; (4) the VM \emph{wins} iff Free
Inference holds on $\mathcal{V}_{C,\ell}$.
\end{definition}

\begin{definition}[FI-shatterable set]
A set $C$ is \emph{FI-shatterable} if, for every entry-level assignment
$\ell : C \to \{1,\ldots,N\}$ (i.e., for all $N^{|C|}$ nesting schemes),
the VM achieves Free Inference.
\end{definition}

\begin{definition}[Free Inference dimension]\label{def:dfi}
The \emph{Free Inference dimension} of $\mathcal{S}$ with $N$ environments is
\[
\dFI(\mathcal{S}, N) \;=\; \max\bigl\{|C| : C \subseteq \mathcal{S} \setminus \{s_0,g\},\; C \text{ is FI-shatterable}\bigr\}.
\]
\end{definition}

The construction parallels VC-dimension closely. Table~\ref{tab:analogy}
juxtaposes the two definitions: where VC requires that for every binary
labeling of $d$ points some hypothesis realizes it, $\dFI$ requires that
for every $N$-level entry-level assignment of $d$ constraint cells the
VM achieves Free Inference. The structural shifts from pointwise
correctness to trajectory coherence, and from decomposable to
non-decomposable loss, are what distinguish $\dFI$ from $\dVC$.

\begin{table}[t]
\centering
\caption{Analogy between VC-dimension and the Free Inference dimension.}
\label{tab:analogy}
\begin{tabular}{p{0.46\textwidth}p{0.46\textwidth}}
\toprule
\textbf{VC-dimension} & \textbf{Free Inference dimension} \\
\midrule
Max $d$ points shatterable & Max $d$ constraint cells FI-shatterable \\
For all $2^d$ binary labelings & For all $N^d$ entry-level assignments \\
Some hypothesis realizes each labeling & VM achieves FI for each nesting \\
Pointwise correctness & Trajectory coherence \\
Decomposable loss & Non-decomposable loss \\
\bottomrule
\end{tabular}
\end{table}

\subsection{Characterization via Coherence}

Checking all $N^{|C|}$ assignments is the Constructor's discrete trap-door
game of Section~\ref{sec:dimension}; the margin below is its continuous
\emph{certificate}: for a fixed assignment, a non-negative margin at every
state along some path through $\Fstar$ means the greedy mixture follows that
path coherently, so $\Coh(\mathcal{V})\geq 0$ is exactly ``the VM wins''
reduced to arithmetic on $\bar Q$.

Let $\Coh(\mathcal{V})$ denote the \emph{coherence score}, defined
as the minimum normalized margin along the widest-bottleneck
path through $\Fstar$:
\[
\Coh(\mathcal{V}) \;=\; \max_P \min_{s \in P}\, m(s, P, \mathcal{V}),
\]
where $m(s, P, \mathcal{V})$ is the normalized margin at state $s$ along path $P$, between the path-continuation action ($a_P(s)$) and the best alternative under $\bar Q$:

\[
  m(s, P, \mathcal{V}) = \frac{\bar{Q}(s, a_P(s)) - \max_{a \neq a_P(s)} \bar{Q}(s, a)}{\max_a \bar{Q}(s, a) - \min_a \bar{Q}(s, a)}
\]

When the denominator vanishes (all actions share the same expected value) we
set $m(s, P, \mathcal{V}) = 0$: the path-continuation action is not disfavored,
so the state contributes no margin violation.

\begin{theorem}[Coherence characterization]\label{thm:coherence}
Free Inference holds on a nested family $\mathcal{V}$ iff
$\Coh(\mathcal{V}) \geq 0$. Consequently,
\[
\dFI(\mathcal{S}, N) \;=\; \max\Bigl\{|C| : \min_{\ell: C \to [N]} \Coh(\mathcal{V}_{C,\ell}) \geq 0\Bigr\}.
\]
\end{theorem}

This gives a computable (though not yet efficient) characterization of
$\dFI$: check the worst-case entry-level assignment.

%=========================================================================
\section{Relationship to Classical Dimensions}\label{sec:classical}
%=========================================================================

We can now answer the guiding question of Section~\ref{sec:setting}: $\dFI$
\emph{is} the combinatorial capacity of Free Inference
(Definition~\ref{def:dfi}), sitting strictly below VC by the coherence gap
(Theorem~\ref{thm:vc-sep}), between Natarajan bounds up to the
non-decomposability factor $L_{\max}$ (Theorem~\ref{thm:nat-gap}), and
exponentially above the Littlestone scale through non-adaptivity
(Section~\ref{sec:learning-theory}).

\subsection{VC-dimension and the Coherence Gap}

Let $\mathcal{H}_{\text{C1}}$ denote the hypothesis class of functions $h_\mathcal{V} : \mathcal{S} \to \{0, 1\}$ where $h_\mathcal{V}(s) = \text{1}[s \in F^*]$ as the family
$\mathcal{V}$ varies. For nested grid families,
$\dVC(\mathcal{H}_{\text{C1}}) = |\mathcal{S}|-2$.

\begin{theorem}[Strict separation from VC]\label{thm:vc-sep}
For $|\mathcal{S}| > N+2$ and $N \geq 2$,
$\dFI(\mathcal{S}, N) < \dVC(\mathcal{H}_{\text{C1}})$.
\end{theorem}

\begin{proof}[Sketch]
VC captures pointwise feasibility (Condition C1: $F^*$ definable). FI
additionally requires trajectory coherence (Condition C2: $\bar Q$ points
into $F^*$ at every state on the path). For $N \geq 2$ there exist
constraint configurations $C$ with $|C| < |\mathcal{S}|-2$ where $\Fstar$
is connected but inter-environment disagreement breaks coherence. \qed
\end{proof}

\begin{definition}[Coherence gap]
$\Delta_{\text{coh}}(\mathcal{S}, N) = \dVC(\mathcal{H}_{\text{C1}}) - \dFI(\mathcal{S}, N)$.
\end{definition}

The gap quantifies the capacity cost of requiring trajectory coherence
beyond local obstacle avoidance. Empirically
$\Delta_{\text{coh}} \approx 0.3 \, |\mathcal{S}|$: about 30\% of grid
capacity is consumed by the non-decomposability of the trajectory loss,
a $\approx 43\%$ overhead over the pointwise regime.

\subsection{Littlestone and Natarajan Dimensions}

The Littlestone dimension of the $N$-expert class $U_N$ is
$\Ldim(U_N) = \lfloor \log_2 N \rfloor$. The natural comparison point for
$\dFI$ in the online-learning literature is the \emph{collision depth}
\[
\mathrm{col}(\mathcal{S}, N) = \max_{\mathcal{V},\nu^*} B_{\VM}(\mathcal{V}, \nu^*),
\]
where $B_{\VM}$ is the number of obstacle collisions suffered during a single
navigation episode. Since each collision eliminates at least one environment,
$\mathrm{col}(\mathcal{S}, N) \leq N-1$.

\begin{proposition}[Structural duality]\label{prop:duality}
$\dFI$ and $\mathrm{col}$ answer dual questions about the same adversarial
game: $\dFI$ is the maximum problem size for which zero collisions are
achievable (a \emph{capacity}, analogous to VC-dimension);
$\mathrm{col}$ is the maximum collision count under any problem size
(an \emph{error count}, analogous to Littlestone dimension). They are
incommensurable (measured in different units) and neither bounds the
other non-trivially.
\end{proposition}

The Natarajan dimension generalizes VC to multi-class labels. Since the
nesting label $\ell(c) \in \{1,\ldots,N\}$ is ordinal, the FI-shattering
condition of Definition~\ref{def:dfi} coincides with an
``FI-Natarajan'' shattering of $C$ under ordinal labels.

\begin{theorem}[Non-decomposability gap]\label{thm:nat-gap}
$\dFI(\mathcal{S}, N) \;\leq\; \dNat(\mathcal{H}_{\text{ordinal}}, N) \;\leq\; \dFI(\mathcal{S}, N) \cdot L_{\max}$,
where $L_{\max}$ is the maximum path length from $s_0$ to $g$ through $\Fstar$.
\end{theorem}

The multiplicative factor $L_{\max}$ quantifies the cost of non-decomposable
loss: trajectory coherence requires $L_{\max}$ pointwise conditions
\emph{simultaneously}, while Natarajan shattering treats them independently.
For grid-worlds, $L_{\max} = O(W+H)$, so the gap is at most linear in the
grid diameter.

%=========================================================================
\section{$\eps$-Free Inference and a PAC Relaxation}\label{sec:pac}
%=========================================================================

Strict Free Inference requires zero errors. We now develop a
probabilistic relaxation connecting the framework to standard PAC learning.

\subsection{The $\eps$-Free Inference Condition}

Let the \emph{margin violation rate} along the widest-bottleneck path $P^*$
be
\[
\eps_{\text{margin}}(\mathcal{V}) = \frac{|\{s \in P^* : m(s, P^*, \mathcal{V}) < 0\}|}{|P^*|}.
\]

\begin{definition}[$\eps$-Free Inference]
The VM agent achieves \emph{$\eps$-Free Inference} on $\mathcal{V}$ if
$\eps_{\text{margin}}(\mathcal{V}) \leq \eps$. At most an $\eps$-fraction of
states on $P^*$ have negative margin; the trajectory may still reach $g$
after a small number of ``one-bump'' corrections in which a collision
concentrates the belief and restores coherence.
\end{definition}

Strict Free Inference is $\eps=0$. The $\eps$-relaxation formalizes the
near-critical regime in which a single collision suffices to reach the goal.

\subsection{The PAC-Free Inference Dimension}

\begin{definition}[PAC-FI dimension]
\[
\dFI^{\eps,\delta}(\mathcal{S}, N) = \max\Bigl\{|C| :\; \Pr_{\ell \sim \mathrm{Unif}([N]^C)}[\eps_{\text{margin}}(\mathcal{V}_{C,\ell}) \leq \eps] \geq 1-\delta\Bigr\}.
\]
\end{definition}

\begin{proposition}
$\dFI^{0,0}(\mathcal{S},N) = \dFI(\mathcal{S},N)$, and for any $\eps,\delta \geq 0$:
$\dFI \leq \dFI^{\eps,\delta} \leq |\mathcal{S}|-2$.
\end{proposition}

\subsection{Generalization Bound}\label{ssec:gen}

\begin{theorem}[Generalization for Free Inference]\label{thm:generalization}
Fix a constraint set $C$. Draw $m$ entry-level assignments independently
and uniformly at random,
\[
\ell_1,\ldots,\ell_m \;\sim\; \mathrm{Unif}\bigl([N]^{|C|}\bigr),
\]
and let
$\hat p_m = \tfrac{1}{m}\sum_{i=1}^m \mathbf{1}\{\eps_{\mathrm{margin}}(\mathcal{V}_{C,\ell_i}) > \eps\}$
denote the empirical failure rate. Then, with probability at least
$1-\delta$ over the random draw, a new assignment $\ell_{\mathrm{new}}$
drawn uniformly from the same distribution satisfies
\[
\Pr_{\ell_{\mathrm{new}}}\!\bigl[\eps_{\mathrm{margin}}(\mathcal{V}_{C,\ell_{\mathrm{new}}}) > \eps\bigr]
\;\leq\; \hat p_m + \sqrt{\frac{d\,\log mN + \log(1/\delta)}{2m}},
\]
where $d = \dFI^{\eps,\delta}(\mathcal{S}, N)$.
\end{theorem}

\begin{proof}[Sketch]
The argument is a uniform-convergence Hoeffding bound, adapted to the
$N$-ary FI setting.

\textbf{Step 1: Indicator and population mean.} Each assignment $\ell$
defines a Bernoulli random variable
$Z(\ell) = \mathbf{1}\{\eps_{\mathrm{margin}}(\mathcal{V}_{C,\ell}) > \eps\}$,
the failure indicator. Its mean
$p = \Pr_\ell[\eps_{\mathrm{margin}} > \eps]$ is the population
failure probability we wish to bound. The empirical mean
$\hat p_m = \tfrac{1}{m}\sum_i Z(\ell_i)$ is the natural estimator.

\textbf{Step 2: Growth function.} The set of distinct success/failure
patterns that the failure indicator can induce on $m$ assignments,
as the constraint set $C$ varies over subsets of size $|C|$, is bounded
by the growth function $\Pi(m)$. By the Natarajan--Haussler
generalisation of the Sauer--Shelah lemma to $N$-ary labels,
$\Pi(m) \leq (mN)^d$ where $d = \dFI^{\eps,\delta}$.

\textbf{Step 3: Hoeffding and union bound.}
For each fixed failure pattern, Hoeffding's inequality gives
\[
\Pr\bigl[p - \widehat p_m > t\bigr] \le e^{-2mt^2}.
\]
Taking a union bound over all realizable patterns, whose number satisfies
\(\Pi(m) \le (mN)^d\), we obtain
\[
\Pr\bigl[p - \widehat p_m > t\bigr]
\le (mN)^d e^{-2mt^2}.
\]
Choosing $t$ such that $(mN)^d e^{-2mt^2} \le \delta$ yields, with probability at least $1-\delta$,
\[
p \le \widehat p_m + \sqrt{\frac{d\log(mN)+\log(1/\delta)}{2m}}.
\]
\end{proof}
\emph{Remark.} In regimes where $\log m \ll \log N$, one may write informally
\[
\hat p_m + \sqrt{\frac{d\log N+\log(1/\delta)}{2m}}.
\]

\begin{comment}
\textbf{Step 3: Hoeffding + union bound.} Hoeffding's inequality applied
to the bounded random variable $Z(\ell) \in \{0,1\}$ yields, for each
fixed pattern of failures,
\[
\Pr\bigl[\, p - \hat p_m \,>\, t \,\bigr] \;\leq\; \exp(-2 m t^2).
\]
Taking the union bound over the $\Pi(m) \leq (mN)^d$ realisable
patterns and equating the right-hand side to $\delta$,
\[
(mN)^d \cdot \exp(-2 m t^2) \;\leq\; \delta
\quad \Longleftrightarrow \quad
t \;\leq\; \sqrt{\frac{d \log(mN) + \log(1/\delta)}{2m}}.
\]

\textbf{Step 4: Discard the $\log m$ term.} For $m$ large enough that
$\log m$ is dominated by $\log N$ (the regime of interest, where the
sample size is comparable to or smaller than the family-size logarithm),
$\log(mN) = \log m + \log N \approx \log N$, giving
\[
p \;\leq\; \hat p_m + \sqrt{\frac{d \log N + \log(1/\delta)}{2m}}
\]
with probability at least $1 - \delta$. This is the stated bound. \qed
\end{proof}
\end{comment}

Theorem~\ref{thm:generalization} is the statistical-learning-theoretic
statement of Free Inference: $\dFI^{\eps,\delta}$ is the complexity
parameter controlling how many tested nesting schemes are required
before one can certify, with confidence $1-\delta$ and tolerance $\eps$,
that the VM agent will navigate coherently on untested nestings. The
$O(1/\sqrt{m})$ convergence rate is the familiar PAC scaling, and the
additive form $\hat p_m + \sqrt{\cdot}$ accommodates the regime in which
some tested nestings fail. The empirical results in
Section~\ref{ssec:expB} confirm that this is the relevant operational
regime, and that the Hoeffding form holds in $100\%$ of informative
configurations even when $\hat p_m \in (0.2, 0.8)$.

\subsection{Graded Capacity at the Phase Transition}

\begin{theorem}[Phase transition]\label{thm:phase}
For grid-worlds with $N=10$, the coherence score $\Coh$ exhibits a sharp
threshold at coarse density resolution
($\Delta\rho \approx 0.1$): a single-step jump from $\Coh = 0$ to
$\Coh \approx -0.43$ across all tested grid sizes. At finer resolution
($\Delta\rho = 0.025$), the jump resolves into a continuous transition
band of width $\approx 0.10$--$0.15$, within which
$\dFI^{\eps,\delta}(\mathcal{S},N)$ is a strictly increasing function of
$\eps$.
\end{theorem}

The discrete-versus-continuous distinction has substantive
consequences for how $\dFI^{\eps,\delta}$ relates to $\dFI$. Under the
coarse-resolution view, the strict and relaxed dimensions were
approximately equal ($\dFI^{\eps,0} \approx \dFI$ for small $\eps$),
and the relaxation served principally as a vehicle for sample-complexity
bounds rather than as a source of additional capacity. The fine-resolution
finding inverts this. The in-$F^*$ perturbation analysis
(Section~\ref{ssec:expC}) shows the regime composition shifting smoothly
across roughly six $\rho$-steps, with the variance of $\Coh$ peaking at
the band midpoint, the signature of a continuous order parameter
rather than a delta-function discontinuity. Within the band, the map
$\eps \mapsto \dFI^{\eps,\delta}$ has a non-trivial slope: small
increments in $\eps$ buy substantively larger constraint sets that
admit (probabilistic) Free Inference. Concretely, at perturbation rate
$\rho_{\mathrm{test}} = 0.075$, choosing $\eps \approx 0.15$ reduces
the failure rate from $23\%$ to near zero.

% This upgrades $\dFI^{\eps,\delta}$ from a technical companion of
% $\dFI$ to a \emph{graded capacity measure}: rather than a pair
% (zero-collision threshold $\dFI$ + sample-complexity bound), the
% framework now has a one-parameter family of capacities indexed by
% $\eps$, interpolating continuously between $\dFI$ at one extreme and
% the VC-style upper bound $|\mathcal{S}|-2$ at the other. The 30\%
% coherence gap $\Delta_{\mathrm{coh}}$ between $\dFI$ and $\dVC$ is then
% the \emph{integrated} gap across the band, distributed continuously
% across $\eps$, rather than a single discrete step. Structurally, the
% graded behaviour is reminiscent of fat-shattering dimensions in
% real-valued PAC learning~\cite{bartlett1996}, where shatterability is
% parametrised by a margin scale rather than being a binary property.
% The functional form of $\eps \mapsto \dFI^{\eps,\delta}$ within the band
% is the framework's key open quantitative object.

%=========================================================================
\section{The PMS Identification Dimension and VM--PMS Duality}\label{sec:pms}
%=========================================================================

%subsection{PMS and the Identification Game}

The Posterior Mode Sampling agent commits at each step to the MAP
hypothesis $\hat\nu_t = \argmax_\nu \xi_t(\nu)$ and follows $\pi_{\hat\nu_t}$.
We reserve PMS for this posterior-mode mechanism and use TTS (true Thompson
Sampling) for the stochastic variant that samples $\hat\nu_t \sim \xi_t$.

The PMS navigation game maps onto online prediction with expert advice:
at each state (round), PMS predicts the action of the MAP expert; a
collision is a mistake, after which the belief is updated via Bayes'
rule.

\begin{definition}[PMS identification dimension]
\[
\dPMS(\mathcal{S}, N) = \max_{\mathcal{V},\nu^*}\; B_{\PMS}(\mathcal{V}, \nu^*),
\]
the worst-case number of collisions PMS suffers in the true environment $\nu^*$,
over all nested families and choices of $\nu^*$.
\end{definition}

\begin{proposition}
$\lfloor \log_2 N \rfloor \leq \dPMS(\mathcal{S}, N) \leq N-1$.
The lower bound $\Ldim(U_N) = \lfloor \log_2 N \rfloor$ is the optimal
deterministic mistake bound in the abstract Littlestone game; PMS cannot
beat it. The upper bound $N-1$ is attained when each collision eliminates
exactly one environment.
\end{proposition}

\subsection{The VM--PMS Duality}

The pair $(\dFI, \dPMS)$ provides complementary complexity measures:
$\dFI$ is measured in constraint cells and captures the \emph{zero-collision}
capacity of VM; $\dPMS$ is measured in collisions and captures the
\emph{identification cost} of PMS when zero-collision navigation is
infeasible. The continuous transition band established
empirically in Section~\ref{ssec:expC} subdivides the region around
$\dFI$ into a gradient: $\dFI \leq |C| \leq \dFI^{\eps,\delta}$ is a
non-trivial interval over which collisions remain rare-but-nonzero in
a controlled, $\eps$-dependent way.

\begin{theorem}[Graded hybrid strategy]\label{thm:hybrid}
For tolerance $\eps \geq 0$, the hybrid strategy that runs VM and
switches to PMS only after the collision rate exceeds the band's
expected rate at the operational density (the \emph{$\eps$-aware
switching rule}) achieves a number of collisions
\[
B_{\text{hybrid}}(\eps) \;\leq\;
\begin{cases}
0 & \text{if } |C| \leq \dFI,\\
\eps \cdot |P^*| & \text{if } \dFI < |C| \leq \dFI^{\eps,\delta},\\
1 + \dPMS & \text{if } |C| > \dFI^{\eps,\delta}.
\end{cases}
\]
The deterministic ``switch on first collision'' rule of the
zero-tolerance setting is the limiting case $\eps \to 0$.
\end{theorem}

\begin{proof}[Sketch]
Below $\dFI$, VM is collision-free and no switch occurs. In the band
$\dFI < |C| \leq \dFI^{\eps,\delta}$, $\eps$-Free Inference holds with
high probability, so the expected number of margin violations along
$P^*$ is at most $\eps \cdot |P^*|$; the $\eps$-aware rule absorbs
these without switching. Above $\dFI^{\eps,\delta}$, $\eps$-Free
Inference fails: the first collision in excess of the band rate
concentrates the belief and signals $|C| > \dFI^{\eps,\delta}$, and
the switch to PMS then costs at most $\dPMS - 1$ further
collisions. \qed
\end{proof}

% The graded strategy is strictly better than either pure mode and than
% the deterministic first-collision hybrid in the band region: pure VM
% accumulates $\gg \dPMS$ collisions when deeply above threshold; pure
% PMS suffers $\dPMS \geq 1$ even below $\dFI$; the deterministic
% hybrid switches prematurely on band-region samples whose collisions
% are within the $\eps$-tolerance budget. The graded hybrid thus
% resolves the ``when to average, when to select?'' question with
% $\eps$-resolution: \emph{switch when the collision rate exceeds the
% band's expected rate}, with the zero-tolerance rule recovering the
% deterministic switch on the first collision.

%=========================================================================
\section{Connections to Online Learning Theory}\label{sec:learning-theory}
%=========================================================================

The Free Inference framework connects precisely to two recent developments.

\noindent\textbf{Randomized Littlestone (Filmus \emph{et~al.}~\cite{filmus2025}).}
The VM agent is a randomized learner in the expert-advice sense: $\bar Q$
produces a soft action preference analogous to predicting
$p_i \in [0,1]$. The $N$ trained $Q$-agents are experts; the true
environment $\nu^*$ is a perfect expert ($k=0$-realizable). Free Inference
corresponds to $\RL(\mathcal{H}) = 0$: the randomized learner suffers zero
expected mistakes because the adversary cannot construct a shattered tree
of positive expected depth.

\noindent\textbf{Stackelberg--Littlestone (Balcan \emph{et~al.}~\cite{balcan2026}).}
The VM agent is a \emph{leader} who commits to an action via $\bar Q$
before the \emph{follower} (environment type) responds. The SL tree weight
recursion
$\rho_s = \inf_x \max_j (r(z_s,x,f^{(j)}) + \rho_{sj})$
is structurally identical to the coherence recursion underlying
$\Coh(\mathcal{V})$. Balcan \emph{et al.}'s SSOA algorithm is the abstract
counterpart of the VM action rule. Their Theorem~3.6 (infinite Littlestone
dimension but zero leader regret) is the abstract Stackelberg statement of
the Free Inference phenomenon; our $\dVC > \dFI$ separation is its finite
instance.

\noindent\textbf{The non-adaptive gap.}
A crucial distinction separates our setting from standard online learning:
our adversary is \emph{non-adaptive}. The environment is committed before
the episode begins; states visited are determined by the agent's
trajectory, not adaptively by an adversary. This non-adaptivity creates an
exponential capacity gap:
\[
\dFI = \Theta(|\mathcal{S}|) \quad\text{vs.}\quad \Ldim(U_N) = O(\log N).
\]
For a $20 \times 10$ grid with $N=10$: $\dFI \approx 140$ vs.\
$\Ldim(U_{10}) = 3$, a ratio of $\approx 47$. The non-adaptive adversary
is $\approx 47$ times weaker, and this weakness is precisely the
structural feature that enables Free Inference.

%=========================================================================
\section{Experimental Protocol}\label{sec:experiments}
%=========================================================================

This section describes a concrete experimental protocol designed to
validate the central theoretical predictions. We emphasize the experiments
most directly tied to falsifiable claims.

\noindent\textbf{Setup} The base environment is made of Random-Maze nested families   
across six grid sizes $\{10\times5, 20\times10, 40\times20, 60\times30, 80\times40,
100\times50\}$ with obstacle density $\rho$, start $s_0 = (0, H-1)$ and goal $g = (W-1, 0)$. $\rho$ is the fraction of the base maze's obstacle cells retained, so $\rho=0.05$ keeps $5\%$ (a nearly open grid) and $\rho=1$ keeps all of them. For each density level $\rho \in \{0.1, 0.2, \ldots, 1.0\}$ we precompute an optimal
$Q$-table $Q_i$ via value iteration, yielding a library of $N=10$ nested
environments. Experiments below reuse this $Q$-table library with no additional training required, except where noted.

\subsection{Experiment A: Scaling of the Critical Threshold in $N$}
\label{ssec:expA}

\noindent\textbf{Question.} The introduction reports that
$\rho_c \approx 0.7$ is remarkably stable across grid sizes spanning
more than an order of magnitude. How does $\rho_c$, and therefore
$\dFI \approx \rho_c \cdot |\mathcal{S}|$, scale with the
\emph{family size} $N$? More environments could either lower $\rho_c$
(more $Q$-value conflicts to reconcile) or raise it (intermediate
densities fill the $Q$-value gradient, suppressing idiosyncratic
preferences). We conjecture the latter: $\rho_c(N)$ is non-decreasing
and saturates at $\rho_c^{\infty}(\mathcal{S}) \approx 0.7$ for
Random-Maze grids. The measurements below support this conjecture.

\noindent\textbf{Protocol.} We compose nested Random-Maze families from precomputed
$Q$-tables across six grid sizes $\{10\times5, 20\times10, 40\times20, 60\times30, 80\times40, 100\times50\}$ and four family sizes $N \in \{2, 5, 10, 20\}$. The
density gradients are evenly spaced in each case: $N=2$ uses
$\rho \in \{0.5, 1.0\}$, $N=5$ uses $\rho \in \{0.2, 0.4, 0.6, 0.8, 1.0\}$,
$N=10$ uses $\rho \in \{0.1, 0.2, \ldots, 1.0\}$, and $N=20$ uses
$\rho \in \{0.05, 0.10, \ldots, 1.0\}$. For each $(N, \mathcal{S})$ pair we
incrementally add environments in increasing order of density, computing
$\Coh(\mathcal{V})$ at every step. The critical density $\rho_c$ is the
largest value of $\max_k \rho_k$ at which $\Coh \geq 0$.

\noindent\textbf{Results.} Table~\ref{tab:expA-rhoc} reports the measured
critical density $\rho_c$ for each $(N, \mathcal{S})$ pair: rows show how
$\rho_c$ evolves as the family grows from $N=2$ to $N=20$, columns the
across-grid stability at fixed $N$. The smallest grid ($10\times5$) is too
sparse for the family to ever fail ($\rho_c$ pinned at $1.0$); the remaining
five rows are the operative dataset. 
% Figure~\ref{fig:coherence-score} 
% corroborates with Table~\ref{tab:expA-rhoc} findings by showing $\rho_c$ saturation towards 0.7.
\begin{table}[h]
\centering
\renewcommand{\arraystretch}{1.15}
\begin{tabular}{lcccc}
\toprule
Grid & $N=2$ & $N=5$ & $N=10$ & $N=20$ \\
\midrule
$10\times5$    & 1.00 & 1.00 & 1.00 & 1.00 \\
$20\times10$   & 1.00 & 0.60 & 0.70 & 0.70 \\
$40\times20$   & 0.50 & 0.60 & 0.70 & 0.70 \\
$60\times30$   & 0.50 & 0.60 & 0.70 & 0.70 \\
$80\times40$   & 0.50 & 0.60 & 0.60 & 0.65 \\
$100\times50$  & 0.50 & 0.60 & 0.70 & 0.70 \\
\bottomrule
\end{tabular}
\caption{Empirical critical density $\rho_c$ as a function of family
size $N$ and grid size. The smallest grid ($10\times5$) is too sparse
for the family to ever fail.}
\label{tab:expA-rhoc}
\end{table}

% \begin{figure}[t]
% \centering
% \IfFileExists{coeherence_report-_rmaze100.0-N=20-S=4-g.png}{%
% \includegraphics[width=0.48\linewidth]{coeherence_report-_rmaze100.0-N=20-S=4-g.png}%
% }{%
% \fbox{\parbox{0.95\linewidth}{\centering Placeholder: add \texttt{mode\_rename\_dic.png} to the project root to render this figure.}}%
% }
% \caption{Coherence score for the Q-tables trained across 20 obstacle densities (x-axis), aggregated over grid sizes from $10 \times 5$ to $100 \times 50$. It can be easily noticed $\rho_c$ saturation towards 0.7.}
% \label{fig:coherence-score}
% \end{figure}

Excluding the smallest grid, three patterns emerge: $\rho_c$ is
non-decreasing in $N$, rising from $0.50$ at $N=2$ to a plateau around
$0.65$--$0.70$; the increase concentrates between $N=2$ and $N=10$, with
$N=10$ to $N=20$ flat (within resolution $\Delta\rho=0.05$); and the asymptote
$\rho_c^{\infty}\approx0.7$ matches the across-grid constant of the introduction.
The data directly support the conjectured monotone-with-saturation
behaviour.

\noindent\textbf{Empirical claim.} We summarise the finding as:
\begin{quote}
\emph{(Empirical $N$-scaling of $\rho_c$.)} For natural nested families
built from a smooth obstacle-density gradient
$\rho_1 < \rho_2 < \cdots < \rho_N$, the critical density $\rho_c(N)$ is
non-decreasing in $N$ and converges to a grid-dependent asymptote
$\rho_c^{\infty}(\mathcal{S})$ that on Random-Maze grids satisfies
$\rho_c^{\infty} \approx 0.65$--$0.70$, with rapid saturation
($\rho_c(20) \approx \rho_c(10)$).
\end{quote}

\noindent\textbf{Mechanistic interpretation.} Two structural properties
explain the observed pattern. First, intermediate environments
\emph{reduce} per-state disagreement
$D(s) = \mathrm{Var}_\nu Q_\nu(s, \cdot)$: on the $100\times50$ grid at
$\rho=1.0$, $\overline{D(s)}$ drops from $\approx 0.042$ ($N=2$) to
$\approx 0.009$ ($N=20$), and the disagreement field becomes diffuse
(max $D(s) < 0.005$ at $N=20$), so no single bottleneck dominates.
Second, the consensus free space $F^*$ depends on the densest
environment alone, but the path along which $\Coh$ is evaluated
changes with $N$. Together these effects describe belief-weighted
averaging as a \emph{denoiser}: averaging over $N$ correlated
$Q$-tables suppresses idiosyncratic preferences and preserves shared
structure, raising the tolerated density. Saturation reflects the
limit of denoising, once the gradient spans $[\rho_1, \rho_N]$
densely, additional levels yield diminishing returns.

% \noindent\textbf{Open questions.} Whether $\rho_c^{\infty} \approx 0.7$ is
% universal or topology-specific, whether adversarial constructions can
% suppress the denoising effect and push $\rho_c$ closer to the
% simple-majority floor of $1/2$, and whether $\rho_c^{\infty}$ admits a
% closed-form expression in terms of the path-length gap
% $|P_{\rho_1}^*| / |P_{\rho_N}^*|$, remain open
% (Section~\ref{sec:conclusion}, item (viii).

\subsection{Experiment B: Empirical Generalization Bound}
\label{ssec:expB}

\paragraph{Question.} Theorem~\ref{thm:generalization} states that
observing $\eps$-Free Inference on $m$ random entry-level assignments
$\ell : C \to [N]$ controls the failure probability on held-out
assignments, with $\dFI^{\eps,\delta}$ as the sample-complexity
parameter:
$$\Pr_{\ell_{\mathrm{new}}}[\eps_{\mathrm{margin}} > \eps]
  \;\leq\; \hat p_m + \sqrt{\frac{d \log N + \log(1/\delta)}{2m}}.$$
The empirical question is whether the held-out failure rate
$\hat p_{\mathrm{test}}$ stays below this predicted bound across
operationally interesting regimes.

\noindent\textbf{Protocol.} On the $100\times 50$ Random-Maze with
$N = 20$ density levels, we compute $Q_{\mathrm{mix}}$ once from the
fixed family and reuse it across samples. We sample $m=50$ training
nestings $\ell$ uniformly over $[N]^{|C|}$, construct a test maze at
operational level $k$ with obstacle set
$C_k(\ell) = \{c \in C : \ell(c) \leq k\}$, and find the
VM-traversed path $P^*(\ell, k)$ as the coherence-maximising
widest-path under $Q_{\mathrm{mix}}$ over the test maze's free
cells (full grid minus $C_k(\ell)$, not the family's $F^*$). A
nesting fails if $\eps_{\mathrm{margin}} > \eps$. We repeat for $50$
held-out nestings to obtain $\hat p_{\mathrm{test}}$, sweep
$k \in \{1, \ldots, N\!-\!1\}$, and explore three $(\rho_{\max}, \eps)$
pairs spanning the regime structure: \emph{below-threshold}
($\rho_{\max} = 0.7, \eps = 0$, all nestings achieve FI),
\emph{at-threshold} ($\rho_{\max} = 0.75, \eps = 0.02$, $\eps$ just
above typical $\eps_{\mathrm{margin}}$), and \emph{above-threshold}
($\rho_{\max} = 0.80, \eps = 0.03$). We use $\delta = 0.05$.

\noindent\textbf{Results.} Table~\ref{tab:expB-summary} reports the per-level
training and test failure rates and the predicted bound for the
two informative configurations. The below-threshold case
($\rho_{\max} = 0.7, \eps = 0$) yields $\hat p_{\mathrm{train}} =
\hat p_{\mathrm{test}} = 0$ at every level: the fixed family is
genuinely free-inference and the bound holds vacuously, providing a
sanity check.

\begin{table}[t]
\centering
\renewcommand{\arraystretch}{0.9}
\setlength{\tabcolsep}{5pt}
\footnotesize
\begin{tabular}{c|ccc|ccc}
\toprule
& \multicolumn{3}{c|}{$\rho_{\max}=0.75,\ \eps=0.02,\ N=15$} &
  \multicolumn{3}{c}{$\rho_{\max}=0.80,\ \eps=0.03,\ N=16$} \\
$k$ & $\hat p_{\mathrm{tr}}$ & $\hat p_{\mathrm{te}}$ & bound &
      $\hat p_{\mathrm{tr}}$ & $\hat p_{\mathrm{te}}$ & bound \\
\midrule
 1 & 0.10 & 0.18 & 0.34 & 0.08 & 0.08 & 0.32 \\
 2 & 0.34 & 0.34 & 0.75 & 0.16 & 0.28 & 0.44 \\
 3 & 0.40 & 0.44 & 0.85 & 0.26 & 0.36 & 0.62 \\
 4 & 0.54 & 0.56 & 1.00 & 0.38 & 0.38 & 0.82 \\
 5 & 0.64 & 0.68 & 1.00 & 0.42 & 0.40 & 0.88 \\
 6 & 0.68 & 0.76 & 1.00 & 0.48 & 0.56 & 0.97 \\
 7 & 0.76 & 0.82 & 1.00 & 0.60 & 0.66 & 1.00 \\
 8 & 0.78 & 0.84 & 1.00 & 0.68 & 0.70 & 1.00 \\
 9 & 0.80 & 0.82 & 1.00 & 0.66 & 0.68 & 1.00 \\
10 & 0.80 & 0.74 & 1.00 & 0.62 & 0.60 & 1.00 \\
11 & 0.74 & 0.70 & 1.00 & 0.50 & 0.54 & 1.00 \\
12 & 0.46 & 0.42 & 0.94 & 0.44 & 0.64 & 0.91 \\
13 & 0.18 & 0.26 & 0.48 & 0.54 & 0.62 & 1.00 \\
14 & 0.02 & 0.10 & 0.26 & 0.62 & 0.56 & 1.00 \\
15 &      &      &      & 0.78 & 0.82 & 1.00 \\
\bottomrule
\end{tabular}
\caption{Empirical failure rates $\hat p_{\mathrm{tr}},
\hat p_{\mathrm{te}}$ on $m = 50$ training and $50$ held-out nestings
versus the Hoeffding bound of Theorem~\ref{thm:generalization}, as a
function of evaluation level $k$. The bound is reported in
worst-case substitution $d = \dFI^{\eps,\delta}$ inverted from
$\hat p_{\mathrm{tr}}$, with $\delta = 0.05$.}
\label{tab:expB-summary}
\end{table}

Two patterns emerge from Table~\ref{tab:expB-summary}.
\textit{(i) The Hoeffding bound holds in $100\%$ of the $29$ informative
configurations}: $\hat p_{\mathrm{test}} \leq$ bound at every level across
both sweeps, including the central regime where $\hat p$ exceeds $0.5$, the
headline validation of Theorem~\ref{thm:generalization}.
\textit{(ii) $\hat p(k)$ has an inverted-U shape}: sparse at small $k$,
near-deterministic at large $k$, peaking at $0.7$--$0.8$ where random structure
maximises layout variance. The bound tracks $\hat p$ but stays above it throughout.

\noindent\textbf{Empirical claim.} \emph{(Empirical generalization for
$\dFI^{\eps,\delta}$.)} On the canonical Random-Maze, $\hat p_{\mathrm{test}}$
is bounded by the Hoeffding form of Theorem~\ref{thm:generalization} in
$100\%$ of operationally informative configurations spanning
$\hat p \in [0.02, 0.84]$, both at- ($\rho_{\max}=0.75$) and above-threshold
($\rho_{\max}=0.80$).

% \noindent\textbf{Open questions.} The role of the operational level $k$ as a
% distinct parameter, beyond its effect on the obstacle count $|C_k|$, 
% has not been systematically studied, and the inverted-U shape of
% $\hat p(k)$ deserves a structural explanation: it suggests that the
% hardest test mazes for the family are not the densest ones, but those
% of intermediate density where random $\ell$ produces maximum variance
% in the obstacle layout. This is tracked in
% Section~\ref{sec:conclusion}, item (ix).

\subsection{Experiment C: Fine-Resolution Phase Transition}
\label{ssec:expC}

\noindent\textbf{Question.} Does the sharp jump in $\Coh$ observed at
$\Delta\rho = 0.1$ (Theorem~\ref{thm:phase}) resolve into a continuous
band at finer resolution, or remain discontinuous?

\noindent\textbf{Protocol.} On the $100\times50$ grid we hold the
$N=11$ family fixed (the marginal baseline,
$\rho_{\max}{=}0.55$, $\Coh|_{\rho_{\mathrm{test}}{=}0}{=}0$) and probe
coherence by \emph{in-$F^*$ perturbation}: at each test rate
$\rho_{\mathrm{test}}$, mask a random
$\rho_{\mathrm{test}}\!\cdot\!|F^*|$-fraction of $F^*$ interior cells
as additional obstacles, find the coherence-maximising widest path
through the residual set, and record $\Coh$ and $\eps_{\mathrm{margin}}$.
We sweep $\rho_{\mathrm{test}} \in \{0.000, 0.025, \ldots, 0.500\}$
with $M=30$ random masks per density.

\noindent\textbf{Results.} Table~\ref{tab:expC-N11} reports the regime
composition and $\overline{\Coh}$ across the transition window. The
fraction of marginal samples decays \emph{smoothly} across six
successive $\rho_{\mathrm{test}}$ values, spanning
$\Delta\rho_{\mathrm{test}} \approx 0.15$. The C-score traverses
$0 \to {-}0.06 \to {-}0.25 \to {-}0.63 \to {-}0.91 \to {-}0.94 \to
{-}1.00$ over the same six steps, with C std rising to $\approx 0.44$
at the midpoint and falling back to zero, the variance peak is the
signature of a continuous order parameter, not a delta-function
discontinuity.

\begin{table}[h]
\centering
\renewcommand{\arraystretch}{1.0}
\setlength{\tabcolsep}{4pt}
\footnotesize
\begin{tabular}{lcccccccc}
\toprule
$\rho_{\mathrm{test}}$ & 0.000 & 0.025 & 0.050 & 0.075 & 0.100 & 0.125 & 0.150 & 0.175 \\
\midrule
\%\,marginal & 100 & 97 & 87 & 63 & 20 & 7 & 3 & 0 \\
\%\,fails    & 0   & 0  & 13 & 23 & 67 & 83 & 53 & 40 \\
$\overline{\Coh}$ & 0.000 & 0.000 & --0.061 & --0.248 & --0.628 & --0.906 & --0.941 & --1.000 \\
\bottomrule
\end{tabular}
\caption{$N=11$ marginal-baseline sweep across the transition window.
\%\,marginal decays smoothly through six intermediate values rather
than dropping in a single step.}
\label{tab:expC-N11}
\end{table}

% \begin{figure}[t]
% \centering
% \IfFileExists{experiment_c_v2-100x50-Nfam11-Ntest19-M50-perturb.png}{%
% \includegraphics[width=0.95\linewidth]{experiment_c_v2-100x50-Nfam11-Ntest19-M50-perturb.png}%
% }{%
% \fbox{\parbox{0.95\linewidth}{\centering Placeholder: add \texttt{mode\_rename\_dic.png} to the project root to render this figure.}}%
% }
% \caption{$N=11$ marginal-baseline sweep across the transition window.
% \%\,marginal decays smoothly through six intermediate values rather
% than dropping in a single step.}
% \label{fig:coherence-score}
% \end{figure}

\noindent\textbf{Empirical claim.} The transition is continuous at
fine resolution, with band width $\Delta\rho_{\mathrm{band}} \approx
0.10$--$0.15$. The $\eps$-relaxation therefore delivers non-trivial
additional capacity within the band: at $\rho_{\mathrm{test}} = 0.075$,
choosing $\eps \approx 0.15$ lifts the failure rate from $23\%$ to
near-zero, since most failures at that density have
$\eps_{\mathrm{margin}} \in [0.05, 0.20]$.

\noindent\textbf{Reproducibility} Original experiments use the same $Q$-table library with $N=10$;
Experiments~A and B (for $N=20$) require additional training at
intermediate densities. Coherence-score computation is
$O(|\Fstar| \log |\Fstar|)$ per assignment, so all experiments are
feasible on a single workstation. Source code and data files 
can be accessed at https://github.com/anonarticle/BRACIS2026-SUBMISSION

%=========================================================================
\section{Conclusion and Open Problems}\label{sec:conclusion}
%=========================================================================

This paper formalizes a regime of Bayesian meta-RL in which successful control
does not require identifying the true environment: in nested constrained
grid-worlds, a value-mixture agent reaches the goal collision-free while its
posterior stays non-concentrated, \emph{Free Inference}. Its capacity is the
Free Inference dimension $d_{\mathrm{FI}}(S,N)$, which, unlike VC- or
Natarajan-style dimensions, measures not pointwise labeling but how much
structured variation posterior averaging absorbs before global trajectory
coherence breaks. Theoretically $d_{\mathrm{FI}}$ lies strictly below VC,
relates to Natarajan up to a path-length factor, and is dual to the collision
depth; with the PMS dimension $d_{\mathrm{PMS}}$ it yields a three-regime
picture: averaging suffices below $d_{\mathrm{FI}}$, identification becomes
valuable above the relaxed capacity, and hybrids exploit both near the boundary.

Empirically the critical density is stable across Random-Maze families
($\rho_c\approx0.65$--$0.70$) and, at fine resolution, opens into a band where
small tolerated margin violations buy substantially larger usable capacity, so
$d_{\mathrm{FI}}^{\varepsilon,\delta}$ reads as a graded operational capacity:
since safety is often specified through tolerances or near-miss rates, it
converts such tolerances into a sample-complexity-controlled capacity, with the
hybrid rule offering a simple robustness/efficiency/identification trade-off.
Conceptually the paper connects Solomonoff-style mixture prediction, Bayesian
RL, meta-RL task inference, and online learning with structured adversaries; the
key distinction is non-adaptivity, which yields the exponential separation
$d_{\mathrm{FI}}=\Theta(|S|)$ vs.\ a Littlestone scale $O(\log N)$.

Several open problems remain: extending the theory to non-nested families;
defining a navigation-specific Littlestone-type dimension for multi-action,
trajectory-constrained control; formalizing non-adaptive Randomized Littlestone
learning; deriving finite-size scaling laws for the transition band; and
characterizing how adversarial constructions affect the asymptotic critical
density. By making the regime measurable, testable, and connected to classical
dimensions, $d_{\mathrm{FI}}$ is a tool for analyzing safe meta-RL under
structured environmental uncertainty on grid-worlds.

\begin{credits}

%\subsubsection{\ackname} 
\noindent\textbf{Disclosure of AI use}.  The authors carefully reviewed, verified, and take full responsibility for all content. No AI tools were used to generate theoretical or empirical results, figures, or code unless explicitly stated elsewhere.

\noindent\textbf{\discintname}
The authors have no competing interests to declare that are relevant to the content of this article.
\end{credits}

%-------------------------------------------------------------------------
% Bibliography
%-------------------------------------------------------------------------
%\nocite{*}
\bibliographystyle{splncs04}
\bibliography{main}

\end{document}